\documentclass{article}

\usepackage{microtype}
\usepackage{graphicx}
\usepackage{subcaption}
\usepackage{booktabs} 

\usepackage{hyperref}

\usepackage[preprint]{icml2026}

\usepackage{amsmath}
\usepackage{amssymb}
\usepackage{mathtools}
\usepackage{amsthm}

\usepackage{bm}
\usepackage{enumerate}
\usepackage{algorithm}
\usepackage{algorithmic}
\usepackage{longtable}
\usepackage[most]{tcolorbox}

\renewcommand{\algorithmicrequire}{\textbf{Input:}}
\renewcommand{\algorithmicensure}{\textbf{Output:}}

\usepackage[capitalize,noabbrev]{cleveref}

\theoremstyle{plain}
\newtheorem{theorem}{Theorem}

\theoremstyle{definition}

\theoremstyle{remark}

\usepackage[textsize=tiny]{todonotes}

\icmltitlerunning{Reinforcement Learning with Conditional Expectation Reward}

\begin{document}

\twocolumn[
  \icmltitle{Reinforcement Learning with Conditional Expectation Reward}



  \icmlsetsymbol{equal}{*}
  \begin{icmlauthorlist}
    \icmlauthor{Changyi Xiao}{yyy}
    \icmlauthor{Caijun Xu}{yyy}
    \icmlauthor{Yixin Cao}{yyy}
  \end{icmlauthorlist}

  \icmlaffiliation{yyy}{Fudan University}

  \icmlcorrespondingauthor{Yixin Cao}{caoyixin2011@gmail.com}

  \icmlkeywords{Machine Learning, ICML}

  \vskip 0.3in
]



\printAffiliationsAndNotice{}  

\begin{abstract}
Reinforcement Learning with Verifiable Rewards (RLVR) has proven effective in enhancing the reasoning capabilities of large language models, particularly in domains such as mathematics where reliable rule-based verifiers can be constructed. However, the reliance on handcrafted, domain-specific verification rules substantially limits the applicability of RLVR to general reasoning domains with free-form answers, where valid answers often exhibit significant variability, making it difficult to establish complete and accurate rules. To address this limitation, we propose Conditional Expectation Reward (CER), which leverages the large language model itself as an implicit verifier, and is therefore applicable to general domains and eliminates the need for external verifiers or auxiliary models. CER is defined as the expected likelihood of generating the reference answer conditioned on the generated answer. In contrast to rule-based verifiers that yield binary feedback, CER provides a soft, graded reward signal that reflects varying degrees of correctness, making it better suited to tasks where answers vary in correctness. Experimental results demonstrate that CER is effective across a wide range of reasoning tasks, spanning both mathematical and general domains, indicating that CER serves as a flexible and general verification mechanism. The code is available at \url{https://github.com/changyi7231/CER}.
\end{abstract}

\section{Introduction}
Reinforcement Learning with Verifiable Rewards (RLVR) has demonstrated strong effectiveness in incentivizing the reasoning capabilities of large language models, which relies on a verifier to provide accurate and checkable reward signals during learning \citep{zhou2025reinforcing}. Such a verifier evaluates the correctness of a generated answer with respect to a given reference answer, and is typically implemented using carefully designed, domain-specific rules that enable deterministic verification \citep{guo2025deepseek,he2025skywork}. RLVR is particularly useful in mathematical reasoning tasks \citep{guo2025deepseek,team2025kimi}, where answers admit canonical or easily normalized representations, allowing correctness to be verified reliably through exact matching or symbolic equivalence checks \citep{math_verify}.

However, RLVR remains difficult to extend to broader reasoning domains such as physics, chemistry, finance, and other domains with open-form answers \citep{ma2025general,zhou2025reinforcing,liu2025nover}. In these domains, valid answers often exhibit diverse surface forms and substantial semantic variation, making it challenging to specify exhaustive verification rules. Consequently, constructing reliable verifiers becomes costly or even infeasible, which substantially constrains the applicability of RLVR to narrowly scoped tasks with well-defined answer spaces.

Moreover, rule-based verifiers typically provide binary feedback, assigning rewards only to strictly equivalent answers while treating all other answers as equally incorrect. As a result, they are unable to assign positive rewards to answers that are partially correct, thereby providing limited learning signals during optimization.

To address these issues, we propose Conditional Expectation Reward (CER) to extend RLVR to general reasoning domains. Rather than relying on external verification rules or auxiliary verifier models, CER uses the large language model itself as an implicit verifier. By exploiting the model’s internal consistency with respect to a reference answer, CER provides a model-intrinsic reward signal that remains applicable even when explicit verification is unavailable.

Specifically, CER measures the expected probability of generating the reference answer conditioned on the model’s generated answer, thereby producing a soft, graded reward signal to verify the generated answer. The underlying intuition is that when a generated answer is identical to, or strongly consistent with, the reference answer, the model will assign a higher conditional probability to reproducing the reference answer given that generation. 

We further show that CER can be theoretically interpreted as a smooth relaxation of the exact-match criterion, yielding reward values that reflect varying degrees of consistency between the generated and reference answers. This characteristic is particularly well suited to general reasoning domains, where partial correctness, semantic overlap, and multiple valid surface realizations are common.

We finally conduct experiments to demonstrate the effectiveness of CER. We show that CER achieves great performance on general domains, both mathematical and non-mathematical domains. These findings highlight CER as a general and robust reward mechanism for RLVR, offering a practical solution for extending reinforcement learning to a wide range of reasoning domains.

\section{Conditional Expectation Reward}
We first introduce RLVR, followed by the definition and theoretical properties of CER. We then present the empirical formulation of CER and the corresponding training objective, and finally describe an efficient procedure for computing CER.
\paragraph{RLVR}
RLVR is a reinforcement learning paradigm in which the reward signal is objectively and automatically checkable by a verifier \citep{zhang2025survey}. Specifically, for a given question $q$ with a unique reference answer $a^*=a^*(q)$, the policy model $\pi_{\theta}$ autoregressively generates a solution $s$ and a final answer $a$ to address the question. Here, the solution $s$ does not include the final answer $a$. Owing to the autoregressive factorization, the policy model $\pi_{\theta}$ satisfies
\begin{align}
    \pi_\theta(a, s|q) = \pi_\theta(s|q) \pi_\theta(a|s,q).
\end{align}
This process yields a quadruple $(q,s,a,a^*)$. An illustrative example is provided in the following box.

\begin{tcolorbox}[
  colback=blue!5!white,  
  colframe=blue!75!black, 
  arc=4pt,               
  boxrule=0.6pt,          
  left=6pt, right=6pt, top=4pt, bottom=4pt,
]
\label{box:prompt}
\textbf{Question $q$:} What is the value of $x$ in the equation $2x + 3 = 7?$\\
\textbf{Solution $s$:} Solve the equation step by step.

First, subtract \(3\) from both sides:
\[
2x + 3 = 7 \;\Rightarrow\; 2x = 4.
\]

Next, divide both sides by \(2\):
\[
x = 2.
\]

Therefore, the value of $x$ is\\
\textbf{Answer $a$:} 2\\
\textbf{Reference answer $a^*$:} 2
\end{tcolorbox}

RLVR is formulated as the optimization of the following expected reward:
\begin{align}
\mathcal{L}_{f}(\theta) = \mathbb{E}_{q\sim \mathcal{D},(s,a)\sim \pi_\theta(\cdot|q)} [f(a,a^*(q))],
\end{align}
where $\mathcal{D}$ is the distribution of $q$, and $f(a, a^{*}(q))$ is a reward that evaluates the correctness of the generated answer $a$ with respect to the reference answer $a^*$ associated with question $q$, which is computed by a verifier $f(\cdot,\cdot)$. In practice, such a verifier is often implemented as a set of carefully designed rules \citep{math_verify}. Rule-based verifiers are particularly effective in domains such as mathematics and code generation, where answers admit unambiguous representations and equivalence can be precisely defined. 

However, extending rule-based verification to general reasoning domains remains challenging \citep{ma2025general,zhou2025reinforcing}. In these domains, valid answers are often free-form and exhibit substantial variation. Consequently, it is difficult to design a rule-based verifier that is both complete and accurate, which limits the direct applicability of RLVR beyond domains with well-structured and formally verifiable answer spaces.

We illustrate this limitation with a concrete example. Consider the following question, for which multiple answers may be semantically equivalent despite differing in surface form. A typical rule-based verifier would only treat $a_1$ as correct, while assigning zero reward to other valid answers such as $a_2$ and $a_3$. This behavior collapses semantically correct but lexically different answers into the same category as incorrect ones, leading to overly sparse and noisy reward signals. Such rigid verification discourages exploration of diverse yet correct answers and hampers effective learning in general reasoning settings.

\begin{tcolorbox}[
  colback=blue!5!white,
  colframe=blue!75!black, 
  arc=4pt,               
  boxrule=0.6pt,          
  left=6pt, right=6pt, top=4pt, bottom=4pt,
]
\label{box:prompt}
\textbf{Question $q$:} Is quantum physics deterministic?\\
\textbf{Reference answer $a^*$:} No\\
\textbf{Answer $a_1$:} No\\
\textbf{Answer $a_2$:} Quantum physics is not deterministic.\\
\textbf{Answer $a_3$:} No, quantum physics is not deterministic; it is probabilistic.
\end{tcolorbox}

\paragraph{Definition}
To generalize RLVR to general domains, we propose the CER, which leverages the policy model itself as an implicit verifier, without relying on external verifiers or models. Instead of explicitly checking answer correctness, CER evaluates the internal consistency of the policy model with respect to a reference answer, thereby enabling applicability to general domains.

For a quadruple $(q,s,a,a^*)$, we define CER as:
\begin{align}
\label{equation:3}
\rho(a, a^*) :=& \mathbb{E}_{s'\sim \pi_\theta(\cdot|q)} \big[ \pi_\theta(a^*|s',q) \, \big| \, A=a \big]\notag\\
=& \mathbb{E}_{s'\sim \pi_\theta(\cdot|q,a)} \big[ \pi_\theta(a^*|s',q) \big].
\end{align}
CER measures the expected likelihood of generating the reference answer $a^*$ given the condition that the model has generated an answer $a$. The intuition is that if the generated answer $a$ is identical to, or strongly correlated with, the reference answer $a^*$, then the policy model should assign a higher conditional probability to generating $a^*$ after having produced $a$.

By optimizing the policy model with respect to CER, the model is encouraged to place higher probability mass on answers that are internally consistent with the reference answer, thereby implicitly guiding the generation process toward $a^*$ without requiring explicit verification.

We illustrate CER with an example from training. For the following given question, the model generates three distinct answers, $\{14, 13, 94\}$. The answer $14$ receives the highest CER value (0.752), as it exactly matches the reference answer. The answer $13$ attains the second-highest CER value (0.313), reflecting its numerical proximity to the reference answer $14$. In contrast, the answer $94$ receives a near-zero CER value (0.00004), as it is numerically distant from the reference answer.

\begin{tcolorbox}[
  colback=blue!5!white,
  colframe=blue!75!black, 
  arc=4pt,               
  boxrule=0.6pt,          
  left=6pt, right=6pt, top=4pt, bottom=4pt,
]
\label{box:prompt}
\textbf{Question $q$:} How many positive multiples of $7$ that are less than $1000$ end with the digit $3$?\\
\textbf{Reference answer $a^*$:} 14\\
\textbf{Generated Answers and CER $(a,\rho(a,a^*))$:} \\$\{(14,0.752),(13,0.313),(94, 0.00004)\}$
\end{tcolorbox}

\paragraph{Properties}
We summarize several fundamental properties of CER, which demonstrate the effectiveness of CER.
\begin{itemize}
    \item \textbf{Boundedness.}
    \[
    0 \le \rho(a,a^*) \le 1.
    \]
    Since $\pi_\theta(a^*|s',q) \in [0,1]$ for all $(q,s')$, the weighted
    sum of these probabilities also lies in $[0,1]$. As a result, CER provides a bounded and well-scaled reward signal that is suitable for stable optimization.
    
    \item \textbf{Minimum.}
    \[
    \begin{aligned}
    \rho(a,a^*) = 0
    \quad\Longleftrightarrow\quad
    \pi_\theta(a^* | s',q) = 0\\
    \quad\text{for all } (q,s') \text{ with } 
    \pi_\theta(s' | q,a) > 0.
    \end{aligned}
    \]
    $\rho(a,a^*)$ is zero exactly when the policy assigns probability~$0$ to $a^*$ on any $(q,s')$ that appear in the conditional distribution defined by $A=a$. In this case, once the policy has generated $a$, it is impossible for the model to regenerate $a^*$ under the same posterior distribution over $(q,s')$.

    \item \textbf{Maximum.}
    \[
    \begin{aligned}
    \rho(a,a^*) = 1
    \quad\Longleftrightarrow\quad
    \pi_\theta(a^* | s',q) = 1\\
    \quad\text{for all } (q,s') \text{ with } 
    \pi_\theta(s' | q,a) > 0.
    \end{aligned}
    \]
    $\rho(a,a^*)$ reaches its maximum value only when the policy
    assigns probability~$1$ to the reference answer $a^*$ for every
    $(q,s')$ that can occur under the conditional distribution defined by $A=a$. In this case, the policy cannot produce any alternative answer with positive probability. Consequently,
    \[
    \rho(a,a^*) = 1
    \quad\Longrightarrow\quad a=a^*
    \]
    
    \item \textbf{Self-Consistency.}
    \begin{theorem}[Exact-Match Case]
    If $a=a^*$, then
    \begin{align*}
        \rho(a^*, a^*)=& \mathbb{E}_{s'\sim \pi_\theta(\cdot|q)} \big[ \pi_\theta(a^*|s',q) \, \big| \, A=a^* \big]\\
        =& \mathbb{E}_{s'\sim \pi_\theta(\cdot|q,a^*)} \big[ \pi_\theta(a^*|s',q) \big]\\
    \;\ge\;&\mathbb{E}_{s\sim \pi_\theta(\cdot|q)} \big[ \pi_\theta(a^*|s,q) \big].
    \end{align*}

    with equality if and only if $\pi_\theta(a^*|s, q)$ is constant over all
    $(q,s)$ such that $\pi_\theta(s |q)>0$.
    \end{theorem}

    See Appendix \ref{appendix} for the proof. This shows that conditioning on the event that the policy has generated
    $a^*$ strictly increases the posterior predictive probability of regenerating
    $a^*$, unless the policy assigns an identical likelihood to $a^*$ across all
    $(q,s)$ pairs with $\pi_\theta(s|q)>0$. Therefore, CER exhibits a self-consistency amplification effect via posterior reweighting toward higher probability assigned to the reference answer in the exact-match case.

    \item \textbf{Equivalence.}

    \begin{theorem}[Value Equivalence]
    \label{theorem:2}
    \begin{align*}
    \mathcal{L}_{\rho}(\theta) &= \mathbb{E}_{q\sim \mathcal{D},(s,a)\sim \pi_\theta(\cdot|q)} [\rho(a,a^*(q))]\\
    &= \mathbb{E}_{q\sim \mathcal{D},(s,a)\sim \pi_\theta(\cdot|q)} [\mathbb{I}(a=a^*(q))],
    \end{align*}
    i.e., the expected CER objective is equivalent in value to the exact-match objective, where $\mathbb{I}(a=a^*(q))$ indicates whether $a$ exactly matches $a^*(q)$.
    \end{theorem}
Thus, $\rho(a, a^*)$ can be interpreted as a soft generalization of the hard exact-match reward $\mathbb{I}(a = a^*)$, while preserving the same expected value. CER yields a continuous-valued reward rather than a binary signal. This property allows CER to provide graded feedback that reflects varying degrees of consistency between the generated answer $a$ and the reference answer $a^*$, which is particularly beneficial in general domains where partial correctness or semantic similarity may be present.
\end{itemize}

In summary, these properties show that CER is a well-behaved and principled reward function. It is bounded and properly scaled, admits clear minimum and maximum conditions, and exhibits a self-consistency amplification effect when the generated answer matches the reference answer. Moreover, CER is value-equivalent in expectation to the exact-match objective while providing a continuous, graded reward signal, thereby serving as a soft generalization of exact-match rewards in general domains.

\paragraph{Empirical CER}
We next derive an empirical form of CER, as the definition in Eq.~(\ref{equation:3}) is intractable due to the summation over all possible outcomes under $\pi_\theta(a^* | s', q)$. To obtain a computable approximation, we apply Bayes’ rule and Monte Carlo sampling to derive an empirical estimator of CER:
\begin{align}
\rho(a, a^*) =& \mathbb{E}_{s'\sim \pi_\theta(\cdot|q)} \big[ \pi_\theta(a^*|s',q) \, \big| \, A=a \big] \notag\\
=& \sum_{s'} \pi_\theta(s'|q,a) \, \pi_\theta(a^*|s',q) \notag\\
= & \frac{\sum_{s'} \pi_\theta(s'|q)\pi_\theta(a|s',q)\pi_\theta(a^*|s',q)}{\sum_{s''} \pi_\theta(s''|q) \pi_\theta(a|s'',q)}\notag\\
\approx & \frac{\sum_{j=1}^{M} \pi_\theta(a|s_j,q)\, \pi_\theta(a^*|s_j,q)}{\sum_{j=1}^{M} \pi_\theta(a|s_j,q)}, s_j \sim \pi_\theta(\cdot|q).
\end{align}
The third line applies Bayes' rule to rewrite $\pi_\theta(s'|q,a)$ in terms of quantities compatible with the autoregressive factorization of large language models.

The final line further applies Monte Carlo estimation by drawing $M$ independent samples $s_j \sim \pi_\theta(\cdot|q)$ to produce an empirical estimate of CER.

The resulting empirical CER is a normalized likelihood-weighted average, where each term $\pi_\theta(a^* | s_j, q)$ is weighted by $\pi_\theta(a | s_j, q)$. This weighting captures the joint consistency of $a$ and $a^*$ under the same conditional context, so that samples for which the policy assigns high probability to both $a$ and $a^*$ contribute more to the estimator, leading to a larger value of CER.

\paragraph{Objective}
We finally define the objective based on the empirical CER. For a quadruple $(q,s,a,a^*)$, the reward function is defined as:
\begin{align}
\label{equation:5}
    &R(q,s,a,a^*)=\frac{\sum_{j=1}^{M} \pi_\theta(a|s_j,q)\pi_\theta(a^*|s_j,q)}{\sum_{j=1}^{M} \pi_\theta(a|s_j,q)},\notag\\
    & \text{where } s_j \sim \pi_\theta(\cdot|q).
\end{align}

Using this reward function, we define the training objective as
\begin{align}
\mathcal{L}_{\rho}(\theta) &= \mathbb{E}_{q\sim \mathcal{D},(s,a)\sim \pi_\theta(\cdot|q)} [\rho(a,a^*)]\notag\\
 &\approx  \mathbb{E}_{q\sim \mathcal{D},(s,a)\sim \pi_\theta(\cdot|q)} [R(q,s,a,a^*)].
\end{align}
Then the corresponding policy gradient is given by
\begin{align}
&\nabla_\theta \mathcal{L}_{\rho}(\theta) \notag\\
\approx &\mathbb{E}_{q\sim \mathcal{D},(s,a)\sim \pi_\theta(\cdot|q)} [R(q,s,a,a^*) \nabla_\theta \log \pi_\theta(a, s|q)]\notag\\
\approx &\mathbb{E}_{q\sim \mathcal{D}} [\frac{1}{N}\sum_{i=1}^{N}R(q,s_{i},a_{i},a^*) \nabla_\theta \log \pi_\theta(a_{i}, s_{i}|q)],\notag\\
 &\text{where } (s_i,a_i) \sim \pi_\theta(\cdot|q).
\end{align}

For each $q$, we sample $N$ independent $(s_i,a_i)$ from $\pi_\theta(\cdot|q)$ for estimating the  gradient. The reward $R(q,s_{i},a_{i},a^*)$ is treated as a fixed scalar with respect to $\theta$ during optimization to detach it from gradient computation for stable policy learning \citep{ziegler2019fine,ouyang2022training}.

\begin{figure*}[t]
\begin{center}
\centerline{\includegraphics[width=0.95\textwidth]{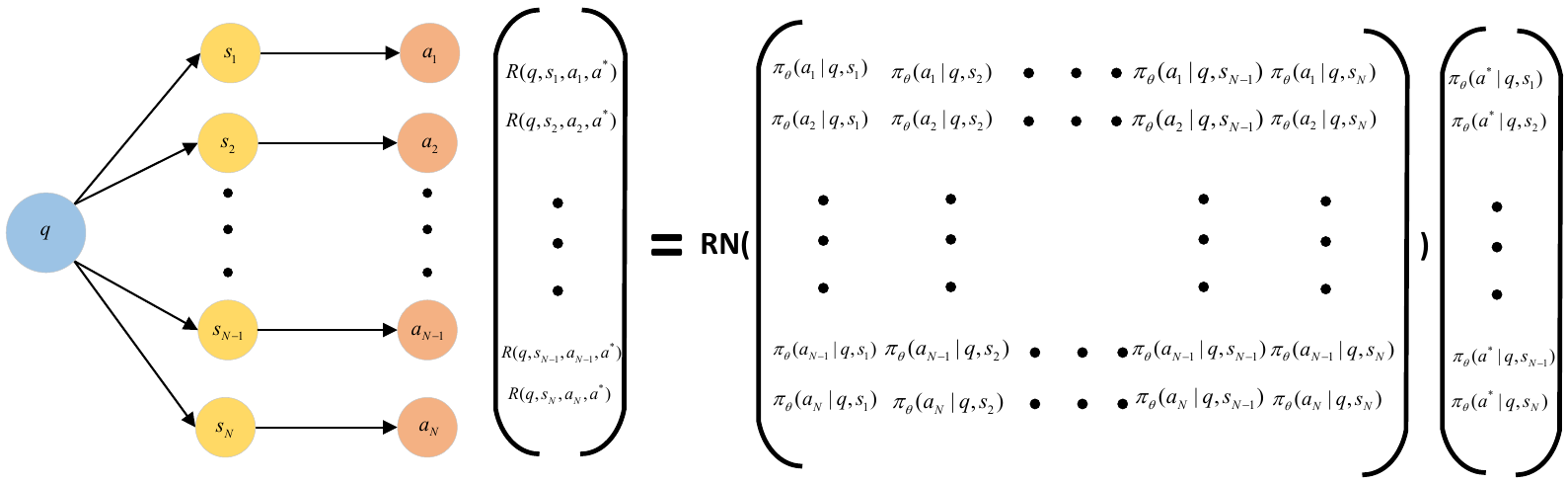}}
\caption{An illustration of CER computation, where RN($\cdot$) denotes row normalization. The left panel depicts the generation process of the quadruple $(q, s_i, a_i, a^*)$, while the right panel shows the CER computation for the quadruple, corresponding to Eq.~(\ref{equation:8}).}
\label{figure:1}
\end{center}
\end{figure*}

\paragraph{Efficiency}
We now describe an efficient procedure for computing CER by reusing samples, avoiding redundant computations and adjusting the hyperparameter. As shown in Eq.~(\ref{equation:5}), computing CER requires sampling $M$ independent solutions $\{s_j\}_{j=1}^{M}$ from $\pi_\theta(\cdot| q)$. However, CER can be seamlessly integrated into policy gradient without additional sampling. Specifically, for each question $q$, we already sample $N$ independent $\{s_i\}_{i=1}^{N}$ from $\pi_\theta(\cdot | q)$ to estimate the policy gradient. These same samples can be directly reused for reward computation by setting $\{s_j\}_{j=1}^{M} := \{s_i\}_{i=1}^{N}$. As a result, CER incurs no extra sampling cost.

To understand the computation of CER more intuitively, we further show that the CER computing can be expressed in a tensorized form. Let $M := N$ and $\{s_j\}_{j=1}^{M} := \{s_i\}_{i=1}^{N}$, and define $\bm{R} \in [0,1]^N$ with entries $\bm{R}_i = R(q, s_i, a_i, a^*)$, $\bm{W} \in [0,1]^{N \times M}$ with entries $\bm{W}_{ij} = \pi_\theta(a_i | s_j, q)$, and $\bm{P} \in [0,1]^M$ with entries $\bm{P}_j = \pi_\theta(a^* | s_j, q)$. The reward vector $\bm{R}$ can then be written as
\begin{align}
\label{equation:8}
\bm{R} = \bm{D}^{-1} \bm{W} \bm{P},
\end{align}
where $\bm{D}$ is a diagonal matrix whose entries are the row sums of $\bm{W}$. See Figure \ref{figure:1} for an illustration.

Although this approach avoids additional sampling, directly computing Eq.~(\ref{equation:8}) still requires $M(N+1)$ forward passes to compute the entries of $\bm{W}$ and $\bm{P}$. We further reduce this cost by eliminating redundant computations. In particular, if two sampled answers $a_{i_1}$ and $a_{i_2}$ are identical, then their corresponding rewards satisfy $R(q, s_{i_1}, a_{i_1}, a^*) = R(q, s_{i_2}, a_{i_2}, a^*)$,
and thus the reward only needs to be computed once for each unique answer.

Besides, the hyperparameter $M$ controls a trade-off between computational efficiency and reward estimation accuracy. Larger values of $M$ yield more accurate estimates of CER at the cost of increased computation, while smaller values improve efficiency with a loss in precision. We can adjust $M$ to achieve a balance between performance and efficiency.

\begin{table*}[!t]
\centering
\begin{minipage}{\linewidth}
    \centering
    \caption{The performance of models trained on a general-domain dataset.}
    \label{table:1}
    \begin{tabular}{lcccccccc}
        \toprule
        \textbf{Methods} & \textbf{MATH500} & \textbf{AMC23} & \textbf{AIME2024} & \textbf{AIME2025} & \textbf{MMLU-Pro} & \textbf{SuperGPQA} & \textbf{Average}\\
        \midrule
        \multicolumn{8}{c}{\textit{Qwen3‑4B‑Base}} \\
        \midrule
        Base                                    &62.6 &40.5 &10.6 & 8.1 &42.0 &21.0 &30.8\\
        Exact-Match                             &78.0 &61.4 &22.3 &20.6 &62.9 &33.5 &46.5\\
        Rule-based                              &84.5 &65.8 &21.7 &18.5 &62.3 &32.6 &47.6\\
        VeriFree                                &83.4 &62.2 &19.6 &16.7 &58.7 &29.1 &44.9\\
        General-verifier                        &84.6 &64.1 &21.9 &17.7 &63.7 &34.2 &47.7\\
        CER                                     &81.6 &\textbf{67.7} &\textbf{22.8} &\textbf{21.3} &63.8 &\textbf{35.2} &48.7\\
        Rule+CER                                &\textbf{85.6} &66.6 &22.5 &19.9 &\textbf{64.1} &\textbf{35.2} &\textbf{49.0}\\
        \midrule
        \multicolumn{8}{c}{\textit{Qwen3‑8B‑Base}}  \\
        \midrule
        Base                                   &73.9 &53.1 &14.6 &12.3 &51.9 &27.0 &38.9\\
        Exact-Match                            &82.4 &66.6 &25.4 &20.4 &66.2 &36.2 &49.5\\
        Rule-based                             &86.0 &72.0 &\textbf{26.7} &21.0 &66.6 &37.7 &51.7\\
        VeriFree                               &86.0 &61.4 &22.3 &19.6 &65.4 &35.5 &48.4\\
        General-verifier                       &84.8 &74.1 &25.0 &\textbf{21.7} &67.3 &37.7 &51.8\\
        CER                                    &\textbf{87.2} &72.3 &25.8 &20.6 &69.7 &\textbf{38.4} &52.3\\
        Rule+CER                               &85.2 &\textbf{76.4} &23.5 &20.6 &\textbf{71.0} &38.3 &\textbf{52.5}\\
        \bottomrule
    \end{tabular}
\end{minipage}

\vspace{1em}

\begin{minipage}{\linewidth}
    \centering
    \caption{The performance of models trained on a mathematical dataset.}
    \label{table:2}
    \begin{tabular}{lcccccccc}
        \toprule
        \textbf{Methods} & \textbf{MATH500} & \textbf{AMC23} & \textbf{AIME2024} & \textbf{AIME2025} & \textbf{MMLU-Pro} & \textbf{SuperGPQA} & \textbf{Average}\\
        \midrule
        \multicolumn{8}{c}{\textit{Qwen3‑4B‑Base}} \\
        \midrule
        Base                                    &62.6 &40.5 &10.6 & 8.1 &42.0 &21.0 &30.8\\
        Exact-match                             &81.2 &57.3 &17.7 &15.6 &46.7 &24.2 &40.5\\
        Rule-based                              &84.2 &63.1 &22.9 &\textbf{21.5} &\textbf{62.5} &\textbf{32.2} &47.7\\
        VeriFree                                &81.5 &62.7 &19.8 &18.1 &59.9 &30.9 &45.5\\
        General-verifier                        &83.6 &63.0 &19.8 &19.0 &58.5 &30.9 &45.7\\
        CER                                     &84.1 &63.6 &\textbf{24.8} &20.0 &60.8 &32.1 &47.6\\
        Rule+CER                                &\textbf{85.0} &\textbf{67.5} &23.3 &20.8 &61.2 &31.3 &\textbf{48.2}\\
        \midrule
        \multicolumn{8}{c}{\textit{Qwen3‑8B‑Base}}  \\
        \midrule
        Base                                    &73.9 &53.1 &14.6 &12.3 &51.9 &27.0 &38.9\\
        Exact-match                             &80.0 &63.4 &16.5 &13.5 &61.4 &33.5 &44.7\\
        Rule-based                              &86.7 &70.2 &26.3 &22.7 &\textbf{65.8} &35.1 &51.1\\
        VeriFree                                &85.0 &68.4 &22.9 &20.2 &62.3 &32.4 &48.5\\
        General-verifier                        &86.0 &69.4 &\textbf{26.7} &20.6 &64.3 &34.9 &50.3\\
        CER                                     &87.2 &70.9 &23.8 &\textbf{23.1} &64.8 &35.0 &50.8\\
        Rule+CER                                &\textbf{87.3} &\textbf{72.0} &26.5 &21.0 &65.6 &\textbf{36.0} &\textbf{51.4}\\
        \bottomrule
    \end{tabular}
\end{minipage}
\end{table*}

\section{Experiments}
We first describe the experimental settings in Section~\ref{section:3.1}. We then present the main experimental results in Section~\ref{section:3.2} to evaluate the effectiveness of the proposed method. Next, we analyze the computational efficiency of CER in Section~\ref{section:3.3}. Finally, in Section~\ref{section:3.4}, we provide a detailed visualization of the CER computation process to offer further insights into its behavior and properties.

\subsection{Settings}
\label{section:3.1}

\paragraph{Datasets}
We evaluate the performance of CER across both mathematical and general reasoning domains. Accordingly, we train models on two datasets: the mathematical dataset MATH-7.5K \citep{hendrycks2021measuring} and the general-domain dataset WebInstruct \citep{ma2025general}. For WebInstruct, we retain only non-mathematical questions at the university difficulty level to focus on general-domain beyond mathematics, yielding a dataset of 50K questions.  WebInstruct includes domains such as physics, chemistry, biology, finance and so on. This subset spans a wide range of disciplines, such as physics, chemistry and biology.

\paragraph{Evaluation}
We evaluate the models on four mathematical datasets, MATH500 \citep{lightman2023let}, AMC23 \citep{aops_amc}, AIME2024 and AIME2025 \citep{aops_aime}, and two general-domain datasets, SuperGPQA \citep{du2025supergpqa} and MMLU-Pro \cite{wang2024mmlu}. Performance is reported using the pass@1 metric. For the mathematical datasets, pass@1 is computed using a rule-based verifier \citep{math_verify}. For the general-domain datasets, which consist of multiple-choice questions, pass@1 is computed via exact matching. For each mathematical dataset, we conduct 16 evaluation runs and report the average performance.

\paragraph{Baselines}
We compare CER with several baseline verifiers. These include the exact-match verifier, which checks whether the generated answer exactly matches the reference answer; a model-based verifier, General-verifier \citep{ma2025general}, which employs an external large language model to assess answer correctness; and a perplexity-based verifier, VeriFree \citep{zhou2025reinforcing}, which uses the perplexity of the reference answer for verification, a rule-based verifier \citep{math_verify}, which verifies the correctness by utilizing handcrafted rules.

\paragraph{Hyperparameter settings}
We set the batch size of questions to $32$, the number of solutions $N$ to $16$, $M$ in Eq.~(\ref{equation:5}) to 16, the learning rate to $10^{-6}$, epoch to 1 for WebInstruct dataset and epoch to 5 for MATH-7.5K dataset. For training, we use temperature $=1.0$ and top-$p=1.0$, while for evaluation we use temperature $=0.6$, top-$p=0.95$ and top-$k=20$. The maximum question length is set to $2048$ tokens, and the maximum output length is set to $4096$ tokens for training and $8192$ for evaluation. We utilize RLOO \citep{kool2019buy} as the optimization method.

\subsection{Results}
\label{section:3.2}
\textbf{CER demonstrates strong generality across domains.}
On the general-domain training dataset (Table \ref{table:1}), CER achieves the highest average performance among all compared methods for both Qwen3-4B-Base and Qwen3-8B-Base (except the combined method Rule+CER). In particular, CER consistently outperforms exact-match rewards and the perplexity-based rewards VeriFree, and exceeds the performance of rule-based verifiers and learned verifiers General-verifier across most evaluation benchmarks. The advantage of CER is especially pronounced on general-domain evaluation datasets such as MMLU-Pro and SuperGPQA, where it achieves consistent performance gains. Notably, this advantage holds without relying on domain-specific handcrafted rules or models.

When trained on the mathematical dataset (Table \ref{table:2}), CER attains performance comparable to rule-based rewards and outperform learned verifier approaches. Despite the absence of an external verifier, CER maintains strong results across mathematical benchmarks, indicating that it does not overfit to a specific domain.

Taken together, these results suggest that CER can serve as a unified reward formulation applicable to both general-domain and mathematics-oriented reasoning tasks.

\textbf{CER substantially improves over exact-match rewards.}
As established in Theorem \ref{theorem:2}, CER can be viewed as a soft generalization of the hard exact-match reward. While exact-match provides a binary signal that only distinguishes perfectly correct answers from all others, CER assigns continuous-valued rewards that reflect partial correctness to the reference answer. Empirically, this difference translates into consistent gains over exact-match training across both datasets and model scales. The graded feedback provided by CER yields denser and more informative learning signals, which is particularly beneficial when correct answers admit surface-level variations.

\textbf{CER is complementary to rule-based rewards.}
We further investigate a simple yet effective strategy for combining CER with rule-based rewards, in which the final reward is defined as the average of the CER score and the rule-based reward. As shown in Tables~\ref{table:1} and~\ref{table:2}, the combined approach (Rule+CER) generally achieves better performance than either method used in isolation. This demonstrates that integrating them yields a more informative training signal.

On the general-domain training dataset, CER enhances rule-based methods by providing graded rewards: rule-based schemes assign positive reward only to strictly equivalent answers and treat all other outputs as equally incorrect. As a result, they fail to differentiate partially correct answers, leading to sparse and uninformative learning signals. CER alleviates this limitation by assigning non-binary rewards that better reflect answer quality.

On the mathematical dataset, rule-based rewards in turn complement CER. In this domain, rule-based methods more reliably capture mathematical equivalence, thereby correcting errors introduced by imperfect similarity estimation in CER. Overall, these results highlight the complementary strengths of CER and rule-based rewards and motivate their combined use across different domains.

\begin{figure*}[!t]
\begin{center}
\centerline{\includegraphics[width=1.0\textwidth]{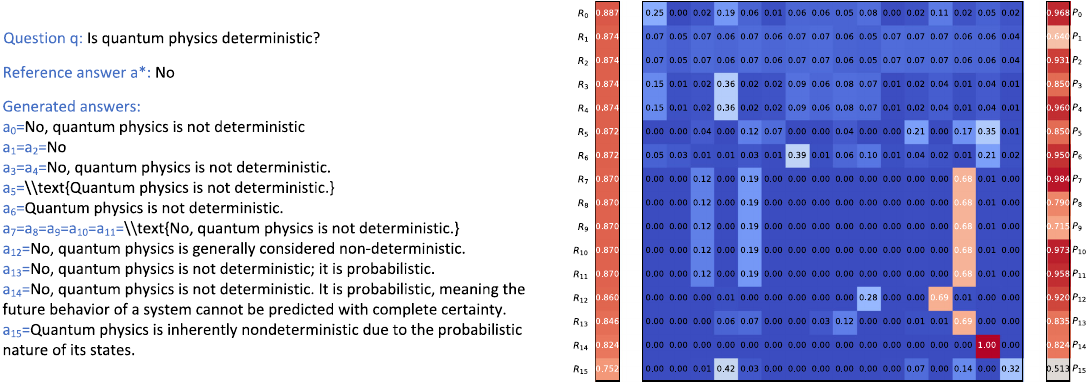}}
\caption{This figure illustrates the computation of CER as defined in Eq.~(\ref{equation:8}). The left panel shows the question, the reference answer, and the 16 generated answers. The right panel depicts the components: the reward vector $\bm{R}$ (left column), the row-normalized matrix $\bm{D}^{-1}\bm{W}$ (central block), and the reference-likelihood vector $\bm{P}$ (right column).}
\label{figure:2}
\end{center}
\end{figure*}

\subsection{Efficiency}
\label{section:3.3}
We analyze the computational efficiency of various rewards and their corresponding performance. In CER, the computational cost is governed by the hyperparameter $M$ in Eq.~(\ref{equation:5}), which specifies the number of samples used to estimate the reward. \textbf{By tuning $M$, CER enables a flexible mechanism to balance runtime efficiency and reward fidelity.}

Table~\ref{table:3} reports runtime and average performance across methods, with all experiments conducted on four NVIDIA H100 GPUs. For CER, increasing $M$ improves performance at the cost of higher runtime overhead. Empirically, CER exhibits a smooth and controllable trade-off, enabling practitioners to select $M$ that balances efficiency and performance under given computational constraints.

Exact-match rewards incur the lowest overhead but yield inferior performance. CER with smaller $M$ and rule-based rewards achieve reasonable performance while remaining efficient, whereas CER with large $M$, VeriFree, and General-verifier incur higher runtime costs due to multiple large language model queries during reward computation.

\begin{table}[!h]
\centering
\caption{The average performance across six datasets and corresponding runtime for each method.}
\label{table:3}
\begin{tabular}{ccc}
\toprule
Model & Performance & Runtime\\
\midrule
Exact-match & 46.5 & 45.2h \\
Rule & 47.6 & 54.7h \\
VeriFree  &44.9  &58.7h \\
General-verifier & 47.7 & 57.5h \\
CER (M=1) & 46.4  & 47.0h \\
CER (M=2) &  47.7 & 52.2h\\
CER (M=4) &  48.0  & 55.6h \\
CER (M=8) &  48.2  & 59.3h  \\
CER (M=16) & 48.7 & 67.4h \\
\bottomrule
\end{tabular}
\end{table}

\subsection{Visualization of CER Computing}
\label{section:3.4}
To better illustrate the computing of CER, we visualize the components involved in Eq.~(\ref{equation:8}). Recall that $\bm{R}_i$ denotes the CER associated with the quadruple $(q, s_i, a_i, a^*)$, $\bm{W}_{ij} = \pi_\theta(a_i | s_j, q)$ represents the likelihood of generating answer $a_i$ conditioned on solution $s_j$ and question $q$, and $\bm{D}$ is the diagonal matrix whose entries are the row sums of $\bm{W}$, $\bm{P}_j = \pi_\theta(a^* | s_j, q)$ measures the likelihood of producing the reference answer $a^*$ given $(s_j, q)$.

Figure~\ref{figure:2} presents a training example. The left panel shows the question $q$, the reference answer $a^*$, and the 16 generated answers $\{a_i\}$. The right panel visualizes the vectors $\bm{R}$ and $\bm{P}$, together with $\bm{D}^{-1}\bm{W}$. Specifically, the left column corresponds to the reward vector $\bm{R}$, the central block depicts the normalized matrix $\bm{D}^{-1}\bm{W}$ with each row summing to one, and the right column corresponds to the vector $\bm{P}$.

Several observations can be drawn from this visualization. \textbf{First, CER effectively captures surface-level variation and semantic similarity among answers, which is particularly important for questions with free-form answers.} In this example, the 16 generated answers contain 10 unique surface forms that nevertheless share similar semantics, such as ``No, quantum physics is generally considered non-deterministic.'' and ``No, quantum physics is not deterministic.'' CER assigns positive rewards to all such semantically consistent answers, whereas exact-match or rule-based methods would assign positive reward only to the strictly matching answer ``No'' This demonstrates that CER provides richer and more informative reward signals for general-domain reasoning tasks.

\textbf{Second, answers that receive higher CER rewards tend to exhibit stronger alignment with solutions that also assign high likelihood to the reference answer}. Since most entries of $\bm{P}$ are relatively large, this alignment can be examined by inspecting the sparsity patterns of the normalized matrix $\bm{D}^{-1}\bm{W}$. From top to bottom, the rows of the normalized matrix become increasingly sparse, which results in smaller CER values. This trend is consistent with the formulation in Eq.~(\ref{equation:5}), where reduced overlap with other solutions leads to a lower reward.

\textbf{Third, the visualization suggests that increasing the value of $M$ in Eq.~(\ref{equation:5}) can improve performance.} A larger $M$ can yield a denser and more stable normalized matrix, leading to a more accurate estimation of CER and, consequently, improved performance.

\textbf{Finally, the figure also illustrates that identical answers receive identical CER rewards.} For example, since $a_1 = a_2$, the corresponding rows in the normalized matrix are identical, which results in equal CER values for $a_1$ and $a_2$. This property reflects the consistency of CER with respect to repeated answer instances.

\section{Related Work}

\paragraph{RLVR}
RLVR \citep{lambert2024tulu,guo2025deepseek,team2025kimi} has emerged as a prominent paradigm for improving the reasoning performance of large language models. RLVR relies on rule-based verifiers to provide accurate and stable reward signals, such as the math-verify library \citep{math_verify} for mathematical reasoning tasks \citep{guo2025deepseek} and the SandboxFusion toolbox \citep{cheng2024fullstack} for code generation \citep{luo2025deepcoder,he2025skywork}. These rule-based verification methods are particularly effective in domains where answers admit unambiguous representations and deterministic equivalence rules can be readily constructed. However, their applicability is limited in general reasoning domains, where valid answers are often open-form and exhibit substantial surface variation. In contrast, CER aims to extend RLVR to such general domains.

\paragraph{General Domains}
Existing verification methods applicable to general reasoning domains can be broadly categorized into model-based verifiers and perplexity-based verifiers.

Model-based verifiers employ a fine-tuned large language model to assess the correctness of a generated answer with respect to a reference answer. For instance, Kimi-k1.5 \citep{team2025kimi} fine-tunes a model on large-scale verification data to endow it with verification capabilities. General-Verifier \citep{ma2025general} further develops a generative, model-based verifier trained specifically for chain-of-thought answer verification, enabling more nuanced and context-aware judgments.

Perplexity-based verifiers, in contrast, define reward signals based on the likelihood or perplexity of the reference answer under the large language model. VeriFree \citep{zhou2025reinforcing} combines perplexity-based rewards with variance reduction techniques to construct a  training objective. Building on this line of work, Nover \citep{liu2025nover} introduces length normalization to mitigate length bias, while RLPR \citep{yu2025rlpr} reformulates perplexity as a sum of token-level probabilities to further address sensitivity to answer length.

In contrast to both model-based and perplexity-based approaches, CER leverages self-consistency between generated answers and the reference answer to produce soft, graded, and model-intrinsic reward signals. This design enables reliable feedback without requiring additional verifier models or handcrafted rules, making CER applicable across a wide range of general reasoning domains.

\section{Conclusion}
In this paper, we propose CER as a general framework for extending RLVR beyond domains that rely on strict, rule-based verification. By leveraging the large language model itself as an implicit verifier, CER produces soft, graded reward signals that reflect partial correctness and semantic consistency, thereby overcoming the limitations of binary rule-based feedback. Our theoretical analysis shows that CER can be viewed as a smooth relaxation of exact-match evaluation, providing a principled connection to conventional verifiable rewards. Empirically, we demonstrate that CER is effective across both mathematical and general-domain reasoning tasks. Together, these results indicate that CER offers a flexible and broadly applicable mechanism for guiding reinforcement learning in large language models, enabling more general and robust reasoning capabilities.

\section*{Impact Statement}
This paper seeks to advance the state of machine learning by introducing new insights and techniques. Progress in this field can enable improvements across many domains, but the work presented here is foundational. Consequently, no specific positive or negative impacts are uniquely attributable to this work.


\bibliography{example_paper}
\bibliographystyle{icml2026}

\newpage
\appendix
\onecolumn

\setcounter{theorem}{0}
\section{Appendix}
\label{appendix}
\subsection{Proof}

\begin{theorem}[Exact-Match Case]
    If $a=a^*$, then
    \begin{align*}
        \rho(a^*, a^*)=& \mathbb{E}_{s'\sim \pi_\theta(\cdot|q)} \big[ \pi_\theta(a^*|s',q) \, \big| \, A=a^* \big]\\
        =& \mathbb{E}_{s'\sim \pi_\theta(\cdot|q,a^*)} \big[ \pi_\theta(a^*|s',q) \big]\\
    \;\ge\;&\mathbb{E}_{s\sim \pi_\theta(\cdot|q)} \big[ \pi_\theta(a^*|s,q) \big].
    \end{align*}

    with equality if and only if $\pi_\theta(a^* | s, q)$ is constant over all
    $(q,s)$ such that $\pi_\theta(s | q)>0$.
\end{theorem}

\begin{proof}
Since $\rho(a,a^*)$ is defined for a generated answer $a$, we always have the probability $\Pr(A=a)>0$. In particular, if $a=a^*$, then $\Pr(A=a^*)>0$, so conditioning on $A=a^*$ is well defined.

By Bayes' rule,
\[
\begin{aligned}
\mathbb{E}_{s'\sim \pi_\theta(\cdot|q)} \!\left[ \pi_\theta(a^*|s',q)| A=a^* \right]
&= \frac{\mathbb{E}_{\, s \sim \pi_\theta(\cdot | q)}\!\left[ \pi_\theta(a^* | s,q) \,\mathbb{I}(A=a^*) \right]}
{\Pr(A=a^*)} \\[4pt]
&= \frac{\mathbb{E}_{\, s \sim \pi_\theta(\cdot | q)}\!\left[ \pi_\theta(a^* | s,q) \,
\Pr(A=a^* | q,s) \right]}
{\mathbb{E}_{ s \sim \pi_\theta(\cdot | q)}\!\left[ \Pr(A=a^* | q,s) \right]} \\[4pt]
&= \frac{\mathbb{E}_{ s \sim \pi_\theta(\cdot | q)}\!\left[ \pi_\theta(a^* | s,q)^2 \right]}
{\mathbb{E}_{ s \sim \pi_\theta(\cdot | q)}\!\left[ \pi_\theta(a^* | s,q) \right]} .
\end{aligned}
\]

Hence,
\[
\rho(a^*,a^*)
= \frac{\mathbb{E}_{s \sim \pi_\theta(\cdot | q)}\!\left[ \pi_\theta(a^* | s,q)^2 \right]}
{\mathbb{E}_{s \sim \pi_\theta(\cdot | q)}\!\left[ \pi_\theta(a^* | s,q) \right]}.
\]

Since $\pi_\theta(a^*|s,q)\ge 0$, Jensen's inequality (or equivalently $\mathbb{E}[X^2]\ge \mathbb{E}[X]^2$) implies
\[
\rho(a^*,a^*)
\;\ge\;
\mathbb{E}_{s\sim \pi_\theta(\cdot|q)} \big[ \pi_\theta(a^*|s,q) \big],
\]
which proves the desired inequality.

Equality holds if and only if $\pi_\theta(a^* | s,q)$ is constant over all $(q,s)$ such that $\pi_\theta(s| q)>0$.

\end{proof}

\begin{theorem}[Value Equivalence]
\begin{align*}
    \mathcal{L}_{\rho}(\theta) &= \mathbb{E}_{q\sim \mathcal{D},(s,a)\sim \pi_\theta(\cdot|q)} [\rho(a,a^*(q))]\\
    &= \mathbb{E}_{q\sim \mathcal{D},(s,a)\sim \pi_\theta(\cdot|q)} [\mathbb{I}(a=a^*(q))],
\end{align*}
i.e., the expected CER objective is equivalent in value to the exact-match objective, where $\mathbb{I}(a=a^*(q))$ indicates whether $a$ exactly matches $a^*(q)$.
\end{theorem}

\begin{proof}
By definition,
\[
\begin{aligned}
\mathcal{L}_{\rho}(\theta)
&= \mathbb{E}_{q\sim\mathcal{D},(s,a)\sim\pi_\theta(\cdot| q)}[\rho(a,a^*(q))] \\
&= \mathbb{E}_{q\sim\mathcal{D}}\!\left[\sum_{s,a} \pi_\theta(s,a|q)\,\rho(a,a^*(q))\right] .
\end{aligned}
\]

Using the definition of $\rho$,
\[
\begin{aligned}
\mathcal{L}_{\rho}(\theta)
&= \mathbb{E}_{q\sim\mathcal{D}}\!\left[\sum_{s,a} \pi_\theta(s,a|q)
    \sum_{s'} \pi_\theta(s'|q,a)\,\pi_\theta(a^*(q)|s',q)\right] \\
&= \mathbb{E}_{q\sim\mathcal{D}}\!\left[\sum_{s'} \pi_\theta(a^*(q)|s',q)
    \sum_{s,a} \pi_\theta(s,a|q)\,\pi_\theta(s'|q,a)\right].
\end{aligned}
\]

For fixed $q$, we have
\[
\sum_{s,a} \pi_\theta(s,a|q)\,\pi_\theta(s'|q,a)
= \sum_a \pi_\theta(a|q)\,\pi_\theta(s'|q,a)
= \pi_\theta(s'|q),
\]
where the first equality marginalizes out $s$ and the second follows from the law of total probability.

Therefore,
\[
\begin{aligned}
\mathcal{L}_{\rho}(\theta)
&= \mathbb{E}_{q\sim\mathcal{D}}\!\left[\sum_{s'} \pi_\theta(s'|q)\,\pi_\theta(a^*(q)|s',q)\right] \\
&= \mathbb{E}_{q\sim\mathcal{D}}\!\left[\sum_{s',a} \pi_\theta(s',a|q)\,\mathbb{I}(a=a^*(q))\right] \\
&= \mathbb{E}_{q\sim\mathcal{D},(s,a)\sim\pi_\theta(\cdot|q)}[\mathbb{I}(a=a^*(q))].
\end{aligned}
\]

This shows that the expected CER objective is equivalent in value to the exact-match objective.
\end{proof}


\end{document}